\documentclass[11pt]{article}
\pdfoutput=1

\usepackage{fullpage,times}

\usepackage[utf8]{inputenc}
\usepackage[T1]{fontenc}
\usepackage{hyperref}
\hypersetup{
	colorlinks   = true,
	urlcolor     = blue,
	linkcolor    = blue,
	citecolor   = black
}
\usepackage{url}
\usepackage{booktabs}
\usepackage{amsmath,amsfonts,amsthm,amssymb}
\usepackage{mathabx}
\usepackage{nicefrac}
\usepackage{microtype}

\usepackage{enumitem}
\usepackage[nolist,nohyperlinks]{acronym}
\usepackage{color}
\usepackage{bold-extra}
\usepackage{anysize}
\usepackage{mathtools}
\usepackage{etoolbox}
\usepackage{wrapfig}

\renewcommand{\P}{\mathbb{P}}

\newcommand{\bA}{\boldsymbol{A}}

\newcommand{\bdelta}{\boldsymbol{\delta}}
\newcommand{\ba}{\boldsymbol{a}}

\newcommand{\sY}{\mathcal{Y}}

\newcommand{\bw}{\boldsymbol{w}}

\newcommand{\bv}{\boldsymbol{v}}

\newcommand{\sX}{\mathcal{X}}

\newcommand{\sD}{\mathcal{D}}

\newcommand{\sH}{\mathcal{H}}

\newcommand{\sZ}{\mathcal{Z}}

\newcommand{\ind}[1]{\mathbb{I}\left\{ #1 \right\}}

\newcommand{\field}[1]{\mathbb{#1}}

\newcommand{\R}{\field{R}}

\newcommand{\scO}{\mathcal{O}}
\newcommand{\sctilO}{\mathcal{\tilde{O}}}
\newcommand{\tilTheta}{\tilde{\Theta}}

\newcommand{\wh}{\widehat}

\newtheorem{lemma}{Lemma}
\newtheorem{theorem}{Theorem}
\newtheorem{cor}{Corollary}

\newtheorem{prop}{Proposition}

\newcommand{\reals}{\mathbb{R}}

\newcommand{\tp}{^{\top}}

\newcommand{\Srep}{{S^{(i)}}}

\newcommand{\repi}{^{(i)}}

\newcommand{\Riskh}{\wh{R}}
\newcommand{\Risk}{R}

\newcommand{\distas}[1]{\mathbin{\overset{#1}{\sim}}}

\newcommand{\iid}{\text{iid}}
\newcommand{\distasiid}{\distas{\iid}}
\newcommand{\ie}{\text{ie.{}}}

\DeclareMathOperator*{\E}{\mathbb{E}}

\newtheorem{definition}{Definition}

\newcommand{\src}{^{\text{src}}}

\newtheoremstyle{named}{}{}{\itshape}{}{\bfseries}{}{.5em}{\thmnote{#3.}#1}
\theoremstyle{named}

\newcommand*\diff{\mathop{}\!\mathrm{d}}

\newcommand{\pr}[1]{\left( #1 \right)}
\newcommand{\br}[1]{\left[ #1 \right]}
\newcommand{\cbr}[1]{\left\{ #1 \right\}}
\newcommand{\abs}[1]{\left|#1\right|}

\newcommand{\lf}{\left}
\newcommand{\rt}{\right}

\newcommand{\Riskinit}{\Risk(\bw_1)}
\newcommand{\GD}{\text{GD}}
\newcommand{\Riskend}{r}
\newcommand{\stabname}{on-average }
\newcommand{\Stabname}{On-Average }

\begin{acronym}
\acro{IC}{Improvement Condition}
\acro{RLS}{Regularized Least Squares}
\acro{TL}{Transfer Learning}
\acro{HTL}{Hypothesis Transfer Learning}
\acro{ERM}{Empirical Risk Minimization}
\acro{TEAM}{Target Empirical Accuracy Maximization}
\acro{RKHS}{Reproducing kernel Hilbert space}
\acro{DA}{Domain Adaptation}
\acro{LOO}{Leave-One-Out}
\acro{HP}{High Probability}
\acro{RSS}{Regularized Subset Selection}
\acro{FR}{Forward Regression}
\acro{PSD}{Positive Semi-Definite}
\acro{SGD}{Stochastic Gradient Descent}
\acro{OGD}{Online Gradient Descent}
\acro{EWA}{Exponentially Weighted Average}
\end{acronym}

\newcommand{\citep}{\cite}

\title{Data-Dependent Stability of Stochastic Gradient Descent}

\author{
  Ilja Kuzborskij \\
  EPFL\\
  Switzerland\\
  \texttt{ilja.kuzborskij@gmail.com} \\
  \and
  Christoph H. Lampert\\
  IST Austria\\
  Klosterneuburg, Austria\\
  \texttt{chl@ist.ac.at}
}

\begin{document}

\maketitle

\begin{abstract} 

We establish a data-dependent notion of algorithmic stability for \ac{SGD}, and employ it to develop novel generalization bounds.
This is in contrast to previous distribution-free algorithmic stability results for \ac{SGD} which depend on the worst-case constants.
By virtue of the data-dependent argument, our bounds provide new insights into learning with \ac{SGD} on convex and non-convex problems.
In the convex case, we show that the bound on the generalization error depends on the risk at the initialization point.
In the non-convex case, we prove that the expected curvature of the objective function around the initialization point has crucial influence on the generalization error.
In both cases, our results suggest a simple data-driven strategy to stabilize \ac{SGD} by pre-screening its initialization.
As a corollary, our results allow us to show optimistic generalization bounds that exhibit fast convergence rates for \ac{SGD} subject to a vanishing empirical risk and low noise of stochastic gradient.

\end{abstract}

\section{Introduction}
\emph{Stochastic gradient descent (\ac{SGD})} has become one 
of the workhorses of modern machine learning.
In particular, it is the optimization method of choice for training 
highly complex and non-convex models, such as neural networks.
When it was observed that these models generalize better 
(suffer less from overfitting) than classical machine learning 
theory suggests, a large theoretical interest emerged to explain
this phenomenon.
Given that \ac{SGD} at best finds a local minimum of the non-convex 
objective function, it has been argued that all such minima might 
be equally good. 
However, at the same time, a large body of empirical work and tricks 
of trade, such as \emph{early stopping}, suggests that in practice one 
might not even reach a minimum, yet nevertheless observes excellent 
performance.

In this work we follow an alternative route that aims to \emph{directly} 
analyze the generalization ability of \ac{SGD} by studying how sensitive 
it is to small perturbations in the training set.
This is known as \emph{algorithmic stability} approach~\cite{BousquetE02}
and was used recently~\cite{hardt2016train} to establish generalization 
bounds for both convex and non-convex learning settings.
To do so they employed a rather restrictive notion of stability that does 
not depend on the data, but captures only intrinsic characteristics of the
learning algorithm and global properties of the objective function.
Consequently, their analysis results in worst-case guarantees that in some
cases tend to be too pessimistic.
As recently pointed out in~\cite{zhang2017understanding}, \emph{deep learning}
might indeed be such a case, as this notion of stability is insufficient 
to give deeper theoretical insights, and a less restrictive one is desirable.

\textbf{As our main contribution} in this work we establish that a data-dependent notion of algorithmic stability, very similar to the \emph{\Stabname Stability}~\cite{shalev2010learnability}, holds for \ac{SGD}
when applied to convex as well as non-convex learning problems.
As a consequence we obtain new generalization bounds that depend  on the data-generating distribution and the initialization point of an algorithm.
For convex loss functions, the bound on the generalization error is essentially multiplicative in the risk at the initialization point when noise of stochastic gradient is not too high.
For the non-convex loss functions, besides the risk, it is also critically controlled by the expected second-order information about the objective function at the initialization point.
We further corroborate our findings empirically and show that, indeed, the data-dependent generalization bound is tighter than the worst-case counterpart on non-convex objective functions.
Finally, the nature of the data-dependent bounds allows us to state \emph{optimistic} bounds that switch to the faster rate of convergence subject to the vanishing empirical risk.

In particular, our findings justify the intuition that \ac{SGD} is more stable
in less curved areas of the objective function and link it to the generalization ability.
This also backs up numerous empirical findings in the deep learning literature that 
solutions with low generalization error occur in less curved regions.
At the same time, in pessimistic scenarios, our bounds are no worse than those 
of~\cite{hardt2016train}.

Finally, we exemplify an application of our bounds, and propose a simple yet 
principled \emph{transfer learning} scheme for the convex and 
non-convex case, which is guaranteed to transfer from the best source 
of information.
In addition, this approach can also be used to select a good initialization 
given a number of random starting positions. This is a theoretically sound 
alternative to the purely random commonly used in non-convex learning.

The rest of the paper is organized as follows.
We revisit the connection between stability and generalization of \ac{SGD} 
in Section~\ref{sec:stability} and introduce a data-dependent notion of stability in Section~\ref{sec:data-stability}.
We state the main results in Section~\ref{sec:results}, 
in particular, Theorem~\ref{thm:sgd_stab_convex} for the convex case, 
and Theorem~\ref{thm:sgd_stab_nonconvex} for the non-convex one.
Next we demonstrate empirically that the bound shown in Theorem~\ref{thm:sgd_stab_nonconvex} is tighter than the worst-case one in Section~\ref{sec:tightness}.
Finally, we suggest application of these bounds by showcasing principled 
transfer learning approaches in Section~\ref{sec:tl},
and we conclude in Section~\ref{sec:conclusions}.

\section{Related Work}
\label{sec:related}
%
%
%
Algorithmic stability has been a topic of interest in learning theory for a long time,
however, the modern approach on the relationship between stability and generalization goes back to the milestone work of~\cite{BousquetE02}.
They analyzed several notions of stability, which fall into two categories: distribution-free and distribution-dependent ones.
The first category is usually called \emph{uniform} stability and focuses on the intrinsic stability properties of an algorithm without regard to the data-generating distribution.
Uniform stability was used to analyze many algorithms, including regularized \ac{ERM}~\cite{BousquetE02}, randomized aggregation schemes~\cite{elisseeff2005stability}, and recently \ac{SGD} by \cite{hardt2016train,london2016generalization}, and \cite{poggio2011online}.
%
%
Despite the fact that uniform stability has been shown to be sufficient to guarantee learnability,
it can be too pessimistic, resulting in worst-case rates.

In this work we are interested in the data-dependent behavior of \ac{SGD}, thus 
the emphasis will fall on the distribution-dependent notion of stability, known as \emph{\stabname} stability, explored throughly in~\cite{shalev2010learnability}.
The attractive quality of this less restrictive stability type is that the resulting bounds are controlled by how stable the algorithm is under the data-generating distribution.
For instance, in~\cite{BousquetE02} and~\cite{devroye1979distribution}, the \stabname stability is related to the variance of an estimator.
In \citep[Sec. 13]{shalev2014understanding}, the authors show risk bounds that depend on the expected empirical risk of a solution to the regularized \ac{ERM}.
In turn, one can exploit this fact to state improved \emph{optimistic} risk bounds, for instance, ones that exhibit \emph{fast-rate} regimes~\cite{koren2015fast,gonen2017fast},
or even to
design enhanced algorithms that minimize these bounds in a data-driven way, e.g.\
by exploiting side information as in transfer~\cite{kuzborskij2013stability,Ben-DavidU13} and metric learning~\cite{perrot2015theoretical}.
Here, we mainly focus on the later direction in the context of \ac{SGD}: how stable is \ac{SGD} under the data-generating distribution given an initialization point?
We also touch the former direction by taking advantage of our data-driven analysis and show optimistic bounds as a corollary.

We will study the \stabname stability of \ac{SGD} for both convex and non-convex loss functions.
In the convex setting, we will relate stability to the risk at the initialization point, while
previous data-driven stability arguments usually consider minimizers of convex \ac{ERM} rather than a stochastic approximation~\citep{shalev2014understanding,koren2015fast}.
Beside convex problems, our work also covers the generalization ability of \ac{SGD} on non-convex problems.
Here, we borrow techniques of~\cite{hardt2016train} and extend them to the distribution-dependent setting.
That said, while bounds of~\cite{hardt2016train} are stated in terms of worst-case quantities, ours reveal new connections to the data-dependent second-order information.
These new insights also partially justify empirical observations in deep learning about the link between the curvature and the generalization error~\cite{hochreiter1997flat,keskar2017large,chaudhuri2017entropy}.
At the same time, our work is an alternative to the theoretical studies of neural network objective functions~\cite{choromanska2015loss,kawaguchi2016deep}, as we focus on the direct connection between the generalization and the curvature.

In this light, our work is also related to non-convex optimization by \ac{SGD}.
Literature on this subject typically studies rates of convergence to the stationary points~\cite{ghadimi2013stochastic,allen2016variance,reddi2016stochastic}, and ways to avoid saddles~\cite{ge2015escaping,lee2016gradient}.
However, unlike these works, and similarly to~\cite{hardt2016train}, we are interested in the generalization ability of \ac{SGD}, and thanks to the stability approach, involvement of stationary points in our analysis is not necessary.



Finally, we propose an example application of our findings in \ac{TL}.
For instance, by controlling the stability bound in a data-driven way, one can choose an initialization that leads to improved generalization.
This is related to \ac{TL} where one transfers from pre-trained models~\cite{kuzborskij2016fast,tommasi2013learning,pentina2014pac,Ben-DavidU13}, especially popular in deep learning due to its data-demanding nature~\cite{galanti2016theoretical}.
Literature on this topic is mostly focused on the \ac{ERM} setting and PAC-bounds, while our analysis of \ac{SGD} yields such guarantees as a corollary.








\section{Stability of \acl{SGD}}
\label{sec:stability}
First, we introduce definitions used in the rest of the paper.
\subsection{Definitions}
\label{sec:definitions}
We will denote with small and capital bold letters respectively
column vectors and matrices, e.g., $\ba=[a_1, a_2, \ldots, a_d]^T\in \R^d~$
and $\bA \in \R^{d_1 \times d_2 }~$,
$\|\ba\|$ is understood as a Euclidean norm and $\|\bA\|_2$ as the spectral norm.
We denote enumeration by $[n] = \{1,\ldots,n\}$ for $n \in \mathbb{N}$.

We indicate an example space by $\sZ$ and its member by $z \in \sZ$.
For instance, in a supervised setting $\sZ = \sX \times \sY$, such that $\sX$ is the input and $\sY$ is the output space of a learning problem.
We assume that training and testing examples are drawn \iid~ from a probability distribution $\sD$ over $\sZ$.
In particular, we will denote the training set as $S=\cbr{z_i}_{i=1}^m \sim \sD^{m}$.

For a parameter space $\sH$, we define a learning algorithm as a map
$
A ~:~ \sZ^m \mapsto \sH
$
and for brevity we will use the notation $A_S = A(S)$.
In the following we assume that $\sH \subseteq \reals^d$.
%
To measure the accuracy of a learning algorithm $A$, we have a \emph{loss} function $f(\bw, z)$, which measures the cost incurred by predicting with parameters $\bw \in \sH$ on an example $z$.
The \emph{risk} of $\bw$, with respect to the distribution $\sD$,
and the \emph{empirical risk} given a training set $S$ are defined as
\[
\Risk(\bw) := \E_{z \sim \sD}[f(\bw, z)], \ \text{ and } \ \Riskh_S(\bw) := \frac{1}{m} \sum_{i=1}^m f(\bw, z_i)~.
\]
Finally, define $\Risk^{\star} := \inf_{\bw \in \sH} \Risk(\bw)$.

%
\subsection{Uniform Stability and Generalization}
On an intuitive level, a learning algorithm is said to be \emph{stable} whenever a small perturbation in the training set does not affect its outcome too much.
Of course, there is a number of ways to formalize the perturbation and the extent of the change in the outcome,
and we will discuss some of them below.
%
The most important consequence of a stable algorithm is that it \emph{generalizes} from the training set to the unseen data sampled from the same distribution.
In other words, the difference between the risk $\Risk(A_S)$ and the empirical risk $\Riskh_S(A_S)$ of the algorithm's output is controlled by the quantity that captures how stable the algorithm is.
So, to observe good performance, or a decreasing true risk, we must have a stable algorithm \emph{and} decreasing empirical risk (training error), which usually comes by design of the algorithm.
In this work we focus on the stability of the \acf{SGD} algorithm, and thus, as a consequence, we study its generalization ability.
%
%
%

Recently,~\cite{hardt2016train} used a stability argument to prove generalization bounds for learning with \ac{SGD}.
%
Specifically, the authors extended the notion of the \emph{uniform stability} originally proposed by~\cite{BousquetE02}, to accommodate randomized algorithms.
\begin{definition}[Uniform stability]
A randomized algorithm $A$ is $\epsilon$-uniformly stable if for all datasets $S, \Srep \in \sZ^m$ such that $S$ and $\Srep$ differ in the $i$-th example, we have
\[
\sup_{z \in \sZ, i \in [m]} \cbr{ \E_{A}\br{f(A_S,z) - f(A_{\Srep},z)} } \leq \epsilon~.
\]
\end{definition}
Since \ac{SGD} is a randomized algorithm, we have to cope with two sources of randomness: the data-generating process and the randomization of the algorithm $A$ itself, hence we have statements in expectation.
The following theorem of~\cite{hardt2016train} shows that the uniform stability implies generalization in expectation.
\begin{theorem}
\label{thm:uniform}
  Let $A$ be $\epsilon$-uniformly stable. Then,
\[
\abs{ \E_{S,A}\br{ \Riskh_S(A_S) - R(A_S) } } \leq \epsilon~.
\]
\end{theorem}
Thus it suffices to characterize the uniform stability of an algorithm to state a generalization bound.
In particular, \cite{hardt2016train} showed generalization bounds for \ac{SGD} under different assumptions on the loss function $f$.
%
%
%
%
%
Despite that these results hold in expectation,
other forms of generalization bounds, such as high-probability ones, can be derived from the above~\cite{shalev2010learnability}.

Apart from~\ac{SGD}, uniform stability has been used before to prove generalization bounds for many learning algorithms~\cite{BousquetE02}.
However, these bounds typically suggest worst-case generalization rates, and rather reflect intrinsic stability properties of an algorithm.
In other words, uniform stability
is oblivious to the data-generating process and any other side information,
which might reveal scenarios where generalization occurs at a faster rate.
In turn, these insights could motivate the design of improved learning algorithms.
%
%
In the following we address some limitations of analysis through uniform stability by using a less restrictive notion of stability.
We extend the setting of~\cite{hardt2016train} by proving data-dependent stability bounds for convex and non-convex loss functions.
In addition, we also take into account the initialization point of an algorithm as a form of supplementary information,
and we dedicate special attention to its interplay with the data-generating distribution.
%
%
Finally, we discuss situations where one can explicitly control the stability of \ac{SGD} in a data-dependent way.


%
%
%
%
\section{Data-dependent Stability Bounds for \acs{SGD}}
\label{sec:data-stability}
In this section we describe a notion of data-dependent algorithmic stability,
that allows us to state generalization bounds which depend not only on the properties of the learning algorithm,
but also on the additional parameters of the algorithm.
We indicate such additional parameters by $\theta$, and therefore we denote stability as a function $\epsilon(\theta)$.
In particular, in the following we will be interested in scenarios where $\theta$ describes the data-generating distribution and the initialization point of \ac{SGD}.
\begin{definition}[\Stabname stability]  
A randomized algorithm $A$ is $\epsilon(\theta)$-\stabname stable if it is true that
\[
\sup_{i \in [m]}\cbr{ \E_{A}\E_{S,z}\br{ f(A_S, z) - f(A_{\Srep}, z)} } \leq \epsilon(\theta)~,
\]
\text{where} $S \distasiid \sD^m$ and $\Srep$ is its copy with $i$-th example replaced by $z \distasiid \sD$.
\end{definition}
Our definition of \stabname stability resembles the notion introduced by~\cite{shalev2010learnability}.
The difference lies in the fact that we take supremum over index of replaced example.
A similar notion was also used by~\cite{BousquetE02} and later by~\cite{elisseeff2005stability} for analysis of a randomized aggregation schemes, however their definition involves absolute difference of losses.
The dependence on $\theta$ also bears similarity to recent work of~\cite{london2016generalization}, however, there, it is used in the context of uniform stability.
The following theorem shows that \stabname-~stable random algorithm is guaranteed to generalize in expectation.
%
%
%
\begin{theorem}
\label{thm:onaverage_stab_generalization}
Let an algorithm $A$ be $\epsilon(\theta)$-\stabname stable. Then,
\[
\E_{S}\E_{A}\br{ \Risk(A_S) - \Riskh_S(A_S) } \leq \epsilon(\theta)~.
\]
\end{theorem}
%
%

\section{Main Results}
\label{sec:results}
Before presenting our main results in this section, we discuss algorithmic details and assumptions.
We will study the following variant of \ac{SGD}:
given a training set $S = \{z_i\}_{i=1}^m \distasiid \sD^m$, step sizes $\cbr{\alpha_t}_{t=1}^T$, random indices $I = \{j_t\}_{t=1}^{T}$, and an initialization point $\bw_1$, perform updates
\begin{align*}
  \bw_{t+1} &= \bw_t - \alpha_t \nabla f(\bw_t, z_{j_t})~
\end{align*}
for $T \leq m$ steps.
Moreover we will use the notation $\bw_{S,t}$ to indicate the output of \ac{SGD} ran on a training set $S$, at step $t$.
%
We assume that the indices in $I$ are sampled from the uniform distribution over $[m]$ \emph{without} replacement, and that this is the only source of randomness for \ac{SGD}.
In practice this corresponds to permuting the training set before making a pass through it, as it is commonly done in practical applications.
We also assume that the variance of stochastic gradients obeys
\begin{equation*}
\E_{S, z}\br{\left\|  \nabla f(\bw_{S,t}, z) - \nabla R(\bw_{S,t}) \right\|^2} \leq \sigma^2 \quad \forall t \in [T]~.
\end{equation*}
%
Next, we introduce statements about the loss functions $f$ used in the following.
\begin{definition}[Lipschitz $f$]
A loss function $f$ is $L$-Lipschitz if $\|\nabla f(\bw, z)\| \leq L$, $\forall \bw \in \sH$ and $\forall z \in \sZ$.
Note that this also implies that
$
|f(\bw, z) - f(\bv, z)| \leq L \|\bw - \bv\|~.
$
\end{definition}
\begin{definition}[Smooth $f$]
A loss function is $\beta$-smooth if $\forall \bw, \bv \in \sH$ and $\forall z \in \sZ$,
$
\|\nabla f(\bw, z) - \nabla f(\bv, z)\| \leq \beta \|\bw - \bv\|~,
$
which also implies
$
f(\bw, z) - f(\bv, z) \leq \nabla f(\bv, z)\tp (\bw - \bv) + \frac{\beta}{2} \|\bw - \bv\|^2~.
$
\end{definition}
\begin{definition}[Lipschitz Hessians]
A loss function $f$ has a $\rho$-Lipschitz Hessian if $\forall \bw, \bv \in \sH$ and $\forall z \in \sZ$,
$
\|\nabla^2 f(\bw, z) - \nabla^2 f(\bv, z)\|_2 \leq \rho \|\bw - \bv\|~.
$
\end{definition}
The last condition is occasionally used in analysis of \ac{SGD}~\cite{ge2015escaping} and holds whenever $f$ has a bounded third derivative.
All presented theorems assume that the loss function used by \ac{SGD} is non-negative, Lipschitz, and $\beta$-smooth.
Examples of such commonly used loss functions are the logistic/softmax losses and neural networks with sigmoid activations.
%
Convexity of loss functions or Lipschitzness of Hessians will only be required for some results, and we will denote it explicitly when necessary.
Proofs for all the statements in this section are given in the supplementary material.
%
\subsection{Convex Losses}
First, we present a new and data-dependent stability result for convex losses.
\begin{theorem}
\label{thm:sgd_stab_convex}
Assume that $f$ is convex, and that \ac{SGD}'s step sizes satisfy $\alpha_t = \frac{c}{\sqrt{t}} \leq \frac{1}{\beta},~ \forall t \in [T]$.
Then \ac{SGD} is $\epsilon(\sD, \bw_1)$-\stabname stable with
\begin{equation*}
  \epsilon(\sD, \bw_1) = \scO\pr{ \sqrt{c\pr{R(\bw_1) - \Risk^{\star}}} \cdot \frac{\sqrt[4]{T}}{m}
  + c \sigma\frac{\sqrt{T}}{m}}~.
\end{equation*}
\end{theorem}
%
%
%
Under the same assumptions, taking step size of order $\scO(1/\sqrt{t})$, \cite{hardt2016train} showed a uniform stability bound $\epsilon = \scO(\sqrt{T/m})$.
Our bound differs since it involves a multiplicative risk at the initialization point.
Thus, our bound corroborates the intuition that whenever we start at a good location of the objective function, the algorithm is more stable and 
thus generalizes better.
However, this is only the case, whenever the variance of stochastic gradient $\sigma^2$ is not too large.
In the extreme case, deterministic case, and of $\Riskinit=0$, the theorem confirms that SGD, in expectation, does not need to make any updates and is therefore perfectly stable.
On the other hand, when the variance $\sigma^2$ is large enough to make the second summand in Theorem~\ref{thm:sgd_stab_convex} dominant, the bound does not offer improvement compared to~\cite{hardt2016train}.
Note, that a result of this type cannot be obtained through the more restrictive uniform stability, precisely because such bounds on the stability must 
hold even for a worst-case choice of data distribution and initialization. 
In contrast, the notion of stability we employ depends on the data-generating distribution, which allowed us to introduce dependency on the risk.

Furthermore, consider that we start at arbitrary location $\bw_1$:
assuming that the loss function is bounded for a concrete $\sH$ and $\sZ$, the rate of our bound up to a constant is no worse than that of \cite{hardt2016train}.
Finally, one can always tighten this result by taking the minimum of two bounds. 
\subsection{Non-convex Losses}
%
Now we state a new stability result for non-convex losses.
\begin{theorem}
\label{thm:sgd_stab_nonconvex}
Assume that $f(\cdot, z) \in [0,1]$ and has a $\rho$-Lipschitz Hessian,
and that step sizes of a form $\alpha_t = \frac{c}{t}$ satisfy
$c \leq \min\cbr{\frac{1}{\beta}, \frac{1}{4 (2 \beta \ln(T))^2}}$.
Then \ac{SGD} is $\epsilon(\sD, \bw_1)$-\stabname stable with
\begin{align}
  \epsilon(\sD, \bw_1) \leq \frac{1 + \frac{1}{c \gamma}}{m} \pr{2 c L^2}^{\frac{1}{1 + c \gamma}} \pr{ \E_{S,A}\br{\Risk(A_S)} \cdot T}^{\frac{c \gamma}{1 + c \gamma}}~,\label{eq:nonconvex}
\end{align}
where
\begin{align}
&\gamma := \sctilO\pr{ \min\cbr{\beta,~ \E_z\br{\lf\|\nabla^2 f(\bw_1, z) \rt\|_2}
 + \Delta_{1,\sigma^2}^{\star}  }}~, \label{eq:gamma}\\
&\Delta_{1,\sigma^2}^{\star} := \rho \pr{ c \sigma + \sqrt{c \pr{\Riskinit - \Risk^{\star}}} }~. \nonumber
\end{align}
\end{theorem}
In particular, $\gamma$ characterizes how the curvature at the initialization point affects stability, and hence the generalization error of \ac{SGD}.
Since $\gamma$ heavily affects the rate of convergence in~\eqref{eq:nonconvex}, and in most situations smaller $\gamma$ yields higher stability, we now look at a few cases of its behavior.
%
%
Consider a regime such that $\gamma$ is of the order $\tilTheta\pr{\E[\|\nabla^2 f(\bw_1, z)\|_2] + \sqrt{\Riskinit} + \sigma}$,
or in other words, that stability is controlled by the curvature, the risk of the initialization point $\bw_1$, and the variance of the stochastic gradient $\sigma^2$.
This suggests that starting from a point in a less curved region with low risk
should yield higher stability, and therefore as predicted by our theory, allow for faster generalization.
In addition, we observe that the considered stability regime offers a principled way to pre-screen a good initialization point in practice, by choosing the one that minimizes spectral norm of the Hessian and the risk.

Next, we focus on a more specific case. Suppose that we choose a step size $\alpha_t = \frac{c}{t}$ such that $\gamma = \tilTheta\pr{\E[\|\nabla^2 f(\bw_1, z)\|_2]}$, yet not too small, so that the empirical risk can still be decreased.
Then, stability is dominated by the curvature around $\bw_1$.
%
%
Indeed, lower generalization errors on non-convex problems, such as training deep neural networks, have been observed empirically when \ac{SGD} is actively guided~\cite{hochreiter1997flat,goodfellow2016deep,chaudhuri2017entropy} or converges to solutions with low curvature~\cite{keskar2017large}.
However, to the best of our knowledge, Theorem~\ref{thm:sgd_stab_nonconvex} is the first to establish a theoretical link between the curvature of the loss function and the generalization ability of \ac{SGD} in a data-dependent sense.

Theorem~\ref{thm:sgd_stab_nonconvex} immediately implies following statement that further reinforces the effect of the initialization point on the generalization error, assuming that $\E_S[R(A_S)] \leq \Riskinit$.
\begin{cor}
\label{cor:non_convex_R1}
Under conditions of Theorem~\ref{thm:sgd_stab_nonconvex} we have that \ac{SGD} is $\epsilon(\sD, \bw_1)$-\stabname stable with
\begin{equation}
\label{eq:non_convex_R1}
    \epsilon(\sD, \bw_1) = \scO\pr{ \frac{1 + \frac{1}{c \gamma}}{m}  \pr{ \Riskinit \cdot T}^{\frac{c \gamma}{1 + c \gamma}} }~.
\end{equation}
\end{cor}
We take a moment to discuss the role of the risk term in $\pr{\Riskinit \cdot T}^{\frac{c \gamma}{1 + c \gamma}}$.
Observe that $\epsilon(\sD, \bw_1) \rightarrow 0$ as $\Riskinit \rightarrow 0$, in other words, the generalization error approaches zero as the risk of the initialization point vanishes.
This is an intuitive behavior, however, uniform stability does not capture this due to its distribution-free nature.
%
Finally, we note that \citep[Theorem 3.8]{hardt2016train} showed a bound similar to~\eqref{eq:nonconvex}, however, in place of $\gamma$ their bound has a Lipschitz constant of the gradient.
The crucial difference lies in term $\gamma$ which is now not merely a Lipschitz constant, but rather depends on the data-generating distribution and initialization point of \ac{SGD}.
We compare to their bound by
considering the worst case scenario, namely, that \ac{SGD} is initialized in a point with high curvature, or altogether, that the objective function is highly curved everywhere.
Then, at least our bound is no worse than the one of~\cite{hardt2016train}, since $\gamma \leq \beta$.
%

Theorem~\ref{thm:sgd_stab_nonconvex} also allows us to prove an optimistic generalization bound for learning with \ac{SGD} on non-convex objectives.
\begin{cor}
Under conditions of Theorem~\ref{thm:sgd_stab_nonconvex} we have that the output of \ac{SGD} obeys
\label{cor:nonconvex_optimistic}
{\small
\begin{align*}
\E_{S,A}\br{\Risk(A_S) - \Riskh_S(A_S)} =
\scO\pr{
\frac{1 + \frac{1}{c \gamma}}{m} \cdot
\max\cbr{ \pr{\E_{S,A}\br{\Riskh_S(A_S)} \cdot T}^{\frac{c \gamma}{1 + c \gamma}},
\pr{\frac{T}{m}}^{c \gamma}
}
}~.
\end{align*}
}
\end{cor}
An important consequence of Corollary~\ref{cor:nonconvex_optimistic}, is that for a vanishing expected empirical risk, in particular for $\E_{S,A}[\Riskh_S(A_S)] = \scO\pr{\frac{T^{c \gamma}}{m^{1 + c \gamma}}}$, the generalization error behaves as $\scO\pr{\frac{T^{c \gamma}}{m^{1 + c \gamma}}}$.
Considering the full pass, that is $m = \scO(T)$,
we have an optimistic generalization error of order $\scO\pr{1/m}$ instead of $\scO(m^{-\frac{1}{1 + c \gamma}})$.
We note that PAC bounds with similar optimistic message (although not directly comparable), but without curvature information can also be obtained through empirical Bernstein bounds as in~\cite{maurer2009empirical}.
However, a PAC bound does not suggest a way to minimize non-convex empirical risk in general,
where, on the other hand, SGD is known to work reasonably well. 
\subsubsection{Tightness of Non-convex Bounds}
\label{sec:tightness}
Next we empirically assess the tightness of our non-convex generalization bounds on real data.
In the following experiment we train a neural network with three convolutional layers interlaced with max-pooling, followed by the fully connected layer with $16$ units, on the MNIST dataset. This totals in a model with $18$K parameters.
%
\begin{figure}
\caption{Empirical tightness of data-dependent and uniform generalization bounds evaluated by training a convolutional neural network.}
\label{fig:mnist_optimistic_nonconvex}
\begin{center}
\includegraphics[width=7cm]{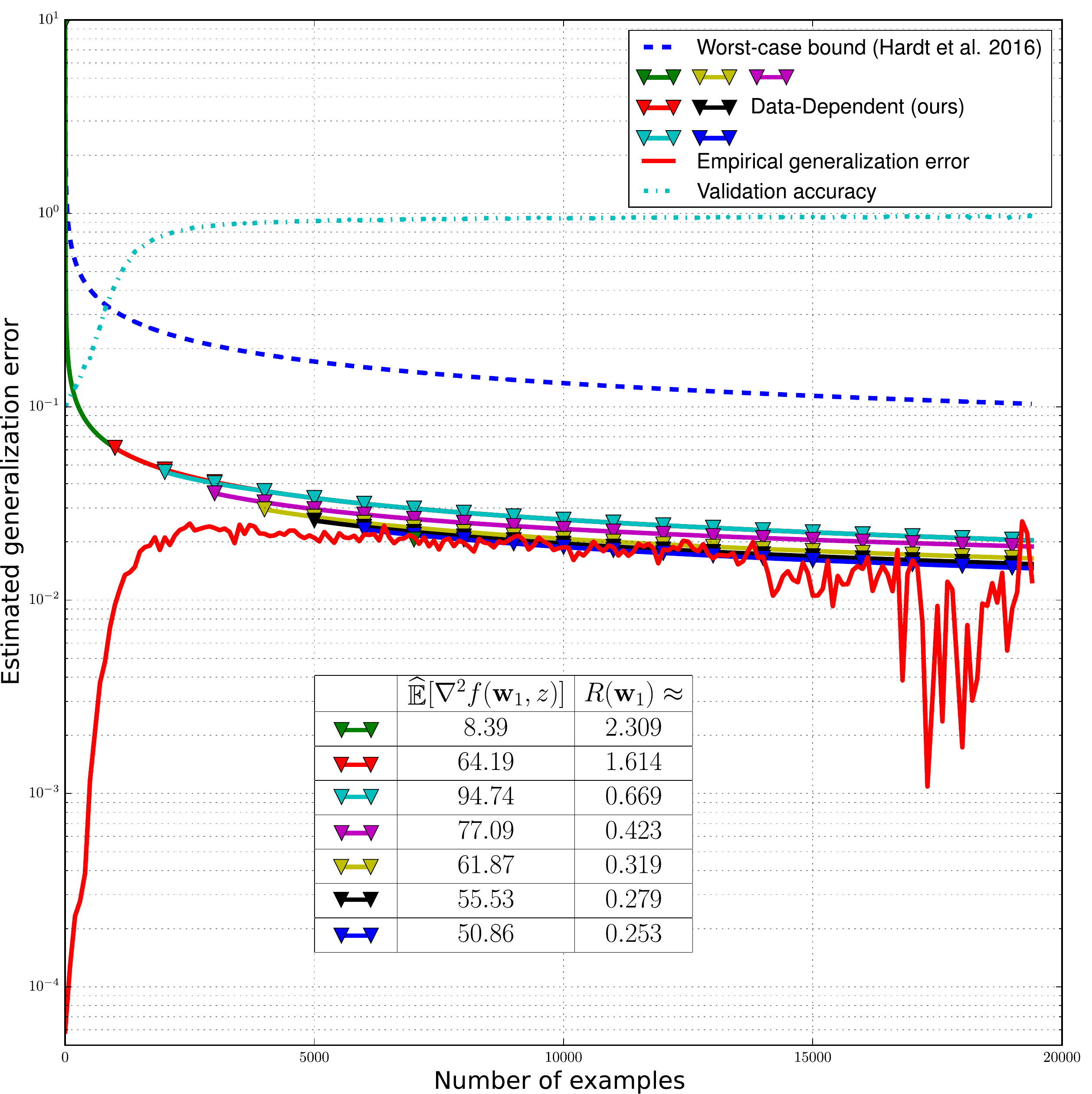}
\end{center}
\end{figure}
Figure~\ref{fig:mnist_optimistic_nonconvex} compares our data-dependent bound~\eqref{eq:nonconvex} to the distribution-free one of~\citep[Theorem 3.8]{hardt2016train}.
As as a reference we also include an empirical estimate of the generalization error taken as an absolute difference of the validation and training average losses.
Since our bound also depends on the initialization point, we plot~\eqref{eq:nonconvex} for multiple ``warm-starts'', \ie with \ac{SGD} initialized from a pre-trained position.
We consider $7$ such warm-starts at every $200$ steps, and report data-dependent quantities used to compute~\eqref{eq:nonconvex} just beneath the graph.
Our first observation is that, clearly, the data-dependent bound gives tighter estimate, by roughly one order of magnitude.
Second, simulating start from a pre-trained position suggests even tighter estimates: we suspect that this is due to decreasing validation error which is used as an empirical estimate for $\Risk(\bw_1)$ which affects bound~\eqref{eq:nonconvex}.

We compute an empirical estimate of the expected Hessian spectral norm by the power iteration method using an efficient Hessian-vector multiplication method~\cite{pearlmutter1994fast}.
Since bounds depend on constants $L$, $\beta$, and $\rho$, we estimate them
by tracking maximal values of the gradient and Hessian norms throughout optimization. 
We compute bounds with estimates $\wh{L} = 78.72$, $\wh{\beta} = 1692.28$, $\wh{\rho} = 3823.73$, and $c = 10^{-3}$.
%
%
%
\subsection{Application to Transfer Learning}
\label{sec:tl}
One example application of data-dependent bounds presented before lies in \emph{\acf{TL}},
where we are interested in achieving faster generalization on a \emph{target} task by exploiting side information that originates from different but related \emph{source} tasks.
The literature on \ac{TL} explored many ways to do so, and here we will focus on the one that is most compatible with our bounds.
%
More formally, suppose that the \emph{target} task at hand is characterized by a joint probability distribution $\sD$, and as before we have a training set $S \distasiid {\sD}^m$.
Some \ac{TL} approaches also assume access to the data sampled from the distributions associated with the \emph{source} tasks.
Here we follow a conservative approach -- instead of the source data,
we receive a set of \emph{source} hypotheses $\cbr{\bw\src_k}_{k=1}^K \subset \sH$, trained on the source tasks.
%
The goal of a learner is to come up with a target hypothesis, which in the optimistic scenario generalizes better by relying on source hypotheses.
%
In the \ac{TL} literature this is known as \ac{HTL}~\cite{kuzborskij2016fast}, that is, we transfer from the source hypotheses which act as a proxy to the source tasks and
the risk $\Risk(\bw\src_k)$ quantifies how much source and target tasks are related.
%
%
In the following we will consider \ac{SGD} for \ac{HTL}, where the source hypotheses act as initialization points.
First, consider learning with convex losses:
Theorem~\ref{thm:sgd_stab_convex} depends on $\Risk(\bw_1)$, thus it immediately quantifies the relatedness of source and target tasks.
So it is enough to pick the point that minimizes the stability bound to transfer from the most related source.
%
Then, bounding $\Risk(\bw\src_k)$ by $\Riskh_S(\bw\src_k)$ through Hoeffding bound along with union bound gives with high probability that
\[
  \min_{k \in [K]} \epsilon(\sD, \bw\src_k) \leq \min_{k \in [K]}  \scO\pr{ \Riskh_S(\bw\src_k) + \sqrt{\frac{\log(K)}{m}} }~.
\]
Hence, the most related source is the one that simply minimizes empirical risk.
%
%
Similar conclusions where drawn in \ac{HTL} literature, albeit in the context of \ac{ERM}.
%
%
Matters are slightly more complicated in the non-convex case.
We take a similar approach, however, now we minimize stability bound~\eqref{eq:non_convex_R1}, and for the sake of simplicity assume that we make a full pass over the data, so $T=m$.
Minimizing the following empirical upper bound select the best source.
\begin{prop}
\label{prop:nonconvex_transfer}
Let
$
\wh{\gamma}_k^{\pm} = \Theta\Big(\frac{1}{m} \sum_{i=1}^m \|\nabla^2 f(\bw\src_k, z_i)\|_2 + \sqrt{\Riskh_S(\bw\src_k)}
 \pm \sqrt[4]{\log(K) / m } \Big)$.
 Then with high probability the generalization error of $\bw\src_k$ is bounded by
\[
\min_{k \in [K]} \scO\pr{ \pr{1 + \frac{1}{c \wh{\gamma}_k^{-}}}
\Riskh_S(\bw\src_k)^{\frac{c \wh{\gamma}^+_k}{1 + c \wh{\gamma}^+_k}}
\cdot
\frac{\sqrt{\log(K)}}{m^{\frac{1}{1 + c \wh{\gamma}^{+}_k}} } }~.
\]
\end{prop}
%
%
Note that $\wh{\gamma}_k^{\pm}$ involves estimation of the spectral norm of the Hessian, which is computationally cheaper to evaluate compared to the complete Hessian matrix~\cite{pearlmutter1994fast}.
This is particularly relevant for deep learning, where computation of the Hessian matrix can be prohibitively expensive.
\section{Conclusions and Future Work}
\label{sec:conclusions}
In this work we proved data-dependent stability bounds for \ac{SGD} and revisited its generalization ability.
We presented novel bounds for convex and non-convex smooth loss functions, partially controlled by data-dependent quantities, while previous stability bounds for \ac{SGD} were derived through the worst-case analysis.
In particular, for non-convex learning, we demonstrated theoretically that generalization of \ac{SGD} is heavily affected by the expected curvature around the initialization point.
We demonstrated empirically that our bound is indeed tighter compared to the uniform one.
In addition, our data-dependent analysis also allowed us to show optimistic bounds on the generalization error of \ac{SGD}, which exhibit fast rates subject to the vanishing empirical risk of the algorithm's output.

In future work we further intend to explore our theoretical findings experimentally and evaluate the feasibility of the transfer learning based on the second-order information.
Another direction lies in making our bounds adaptive. So far we have presented bounds that have data-dependent components, however the step size cannot be adjusted depending on the data, e.g.\ as in~\cite{zhao2015stochastic}.
This was partially addressed by~\cite{london2016generalization}, albeit in the context of uniform stability,
and we plan to extend this idea to the context of data-dependent stability.


\bibliographystyle{plain}
\bibliography{learning}

 \section*{Acknowledgments}
This work was in parts funded by the European Research Council (ERC)
under the European Union's Horizon 2020 research and innovation programme
(grant agreement no 637076).\\~\\
This work was in parts funded by the European Research Council under the
European Union's Seventh Framework Programme (FP7/2007-2013)/ERC grant
agreement no 308036.

\appendix










%
\section{Proofs}
\label{sec:proofs}
In this section we present proofs of all the statements.
\begin{proof}[Proof of Theorem~\ref{thm:onaverage_stab_generalization}]
Indicate by $S=\{z_i\}_{i=1}^m$ and $S'=\{z'_i\}_{i=1}^m$ independent training sets sampled i.i.d.\ from $\sD$, and let $\Srep=\cbr{z_1, \ldots, z_{i-1}, z_i', z_{i+1}, \ldots, z_m}$, such that $z_i' \distasiid \sD$.
We relate expected empirical risk and expected risk by
  \begin{align*}
    \E_S \E_A\br{ \Riskh_S(A_S) }
=   &\E_S \E_A\br{ \frac{1}{m} \sum_{i=1}^m f(A_S, z_i) }\\
=   &\E_{S,S'} \E_A\br{ \frac{1}{m} \sum_{i=1}^m f(A_{\Srep}, z'_i) }\\
=   &\E_{S,S'} \E_A\br{ \frac{1}{m} \sum_{i=1}^m f(A_S, z'_i) } - \delta\\
=   &\E_{S} \E_A\br{ \Risk(A_S) } - \delta~,
  \end{align*}
where
\begin{align*}
  \delta = &\E_{S,S'} \E_A\br{ \frac{1}{m} \sum_{i=1}^m \pr{ f(A_S, z'_i) - f(A_{\Srep}, z'_i) } }\\
  = &\frac{1}{m} \sum_{i=1}^m \E_{S, z'_i} \E_A\br{ f(A_S, z'_i) - f(A_{\Srep}, z'_i) }~.
\end{align*}
Renaming $z'_i$ as $z$ and taking $\sup$ over $i$ we get that
\begin{align*}
\delta \leq \sup_{i \in [m]}\cbr{ \E_{S, z} \E_A\br{ f(A_S, z) - f(A_{\Srep}, z) } }~.
\end{align*}
This completes the proof.
\end{proof}
\subsection{Preliminaries}
%
We say that the \ac{SGD} gradient update rule is an operator $G_t~:~\sH~\mapsto~\sH$, such that
\[
  G_t(\bw) := \bw - \alpha_t \nabla f(\bw, z_{i_t})~,
\]
and it is also a function of the training set $S$ and a random index set $I$.
Then, $\bw_{t+1} = G_t(\bw_t)$, throughout $t = 1, \ldots, T$.
Recall the use of notation $\bw_{S,t}$ to indicate the output of \ac{SGD} ran on a training set $S$, at step $t$,
and define
\[
\delta_t(S,z) := \|\bw_{S,t} - \bw_{\Srep,t}\|~.
\]
%
%
Next, we summarize a few instrumental facts about $G_t$ and few statements about the loss functions used in our proofs.
%
\begin{definition}[Expansiveness]
A gradient update rule is $\eta$-expansive if for all $\bw, \bv$,
\[
\|G_t(\bw) - G_t(\bv)\| \leq \eta \|\bw - \bv\|~.
\]
\end{definition}
The following lemma characterizes expansiveness for the gradient update rule under different assumptions on $f$.
\begin{lemma}[Lemma 3.6 in \citep{hardt2016train}]
  \label{lem:worst-case-expansive}
  Assume that $f$ is $\beta$-smooth. Then, we have that:
\begin{enumerate}
  \item[1)] $G_t$ is $(1+\alpha_t \beta)$-expansive,
  \item[2)] If $f$ in addition is convex, then, for any $\alpha_t \leq \frac{2}{\beta}$, the gradient update rule $G_t$ is $1$-expansive.
\end{enumerate}
\end{lemma}
An important consequence of $\beta$-smoothness of $f$ is self-boundedness~\cite{shalev2014understanding}, which we will use on many occasions.
\begin{lemma}[Self-boundedness]
\label{lem:self_boundedness}
For $\beta$-smooth non-negative function $f$ we have that
\[
\|\nabla f(\bw, z)\| \leq \sqrt{2 \beta f(\bw, z)}~.
\]
\end{lemma}
Self-boundedness in turn implies the following boundedness of a gradient update rule.
\begin{cor}
\label{cor:G_boundedness}
Assume that $f$ is $\beta$-smooth and non-negative.
Then,
\[
\|\bw - G_t(\bw)\| = \alpha_t \|\nabla f(\bw, z_{j_t})\| \leq \alpha_t \min\cbr{ \sqrt{2 \beta f(\bw, z_{j_t})}, L }~.
\]
\end{cor}
%
%
\begin{proof}
By Lemma~\ref{lem:self_boundedness}
  \begin{align*}
    \|\alpha_t \nabla f(\bw, z_{j_t})\| \leq \alpha_t \sqrt{2 \beta f(\bw, z_{j_t})}~,
  \end{align*}
and also by Lipschitzness of $f$, $\|\alpha_t \nabla f(\bw, z_{j_t})\| \leq \alpha_t L$.
\end{proof}
%
%
%
Next we introduce a bound that relates the risk of the output at step $t$ to the risk of the initialization point $\bw_1$ through the variance of the gradient.
Given an appropriate choice of step size, this bound will be crucial at stating stability bounds that depend on the risk at $\bw_1$.
%
The proof idea is similar to the one of~\cite{ghadimi2013stochastic}.
In particular, it does not require convexity of the loss function.
%
\begin{lemma}
\label{lem:stationary_point}
Suppose \ac{SGD} is ran with step sizes $\alpha_1, \ldots, \alpha_{t-1} \leq \frac{1}{\beta}$ w.r.t.\ the $\beta$-smooth loss $f$.
Then we have that
  \begin{align}
  \sum_{k=1}^{t-1} \pr{\alpha_k - \frac{\alpha_k^2 \beta}{2}} \E_S\br{ \|\nabla R(\bw_k)\|^2 }
\leq R(\bw_1) - R(\bw_t) + \frac{\beta}{2} \sum_{k=1}^{t-1} \alpha_k^2 \E_S\br{ \|\nabla f(\bw_k, z_{j_k}) - \nabla R(\bw_k)\|^2 }~.
\end{align}
\end{lemma}
\begin{proof}
  For brevity denote $f_k(\bw) \equiv f(\bw, z_{j_k})$.
By $\beta$-smoothness of $R$ and recalling that the \ac{SGD} update rule $\bw_{k+1} = \bw_k - \alpha_k \nabla f_k(\bw_k)$, we have
\begin{align*}
  R(\bw_{k+1}) - R(\bw_k) &\leq \nabla R(\bw_k)\tp \pr{\bw_{k+1} - \bw_k} + \frac{\beta}{2} \|\bw_{k+1} - \bw_k\|^2 \\
  &= - \alpha_k \nabla R(\bw_k)\tp \nabla f_k(\bw_k) + \frac{\beta \alpha_k^2}{2} \|\nabla f_k(\bw_k)\|^2 \\
  &= - \alpha_k \nabla R(\bw_k)\tp \nabla f_k(\bw_k) + \frac{\beta \alpha_k^2}{2} \|\nabla f_k(\bw_k) - \nabla R(\bw_k) + \nabla R(\bw_k)\|^2 \\
  &= - \alpha_k \nabla R(\bw_k)\tp \nabla f_k(\bw_k)\\
&+ \frac{\beta \alpha_k^2}{2} \Big( \|\nabla f_k(\bw_k) - \nabla R(\bw_k)\|^2 + \|\nabla R(\bw_k)\|^2\\
&- 2 \pr{\nabla f_k(\bw_k) - \nabla R(\bw_k)}\tp \nabla R(\bw_k) \Big)\\
&=- \pr{\alpha_k + \alpha_k^2 \beta} \nabla R(\bw_k)\tp \nabla f_k(\bw_k)\\
&+ \frac{3 \alpha^2 \beta}{2} \|\nabla R(\bw_k)\|^2 + 
\frac{\beta \alpha_k^2}{2} \|\nabla f_k(\bw_k) - \nabla R(\bw_k)\|^2~.
\end{align*}
Taking expectation w.r.t. $S$ on both sides, recalling that $\E_{z_k}[\nabla f_k(\bw_k)] = \nabla R(\bw_k)$ and rearranging terms we get
\[
\pr{\alpha_k - \frac{\alpha^2 \beta}{2}} \E\br{ \|\nabla R(\bw_k)\|^2 } \leq R(\bw_k) - R(\bw_{k+1}) + \frac{\beta \alpha_k^2}{2} \E\br{ \|\nabla f_k(\bw_k) - \nabla R(\bw_k)\|^2 }~,
\]
and summing above over $k = 1, \ldots, t-1$ we get the statement.
\end{proof}
%
%
%
\begin{lemma}
  \label{lem:E_path_bounded_by_R1_and_noise}
Suppose \ac{SGD} is ran with step sizes $\alpha_1, \ldots, \alpha_{t-1} \leq \frac{1}{\beta}$ on the $\beta$-smooth loss $f$.
Assume that the variance of stochastic gradients obeys
\[
\E_{S, z}\br{\left\|  \nabla f(\bw_{S,k}, z) - \nabla R(\bw_{S,k}) \right\|^2} \leq \sigma^2 \quad \forall k \in [T]~.
\]
Then we have that
  \begin{equation*}
    \E_S\br{\sum_{k=1}^{t-1} \alpha_k \|\nabla f(\bw_{S,k}, z_k)\| }
\leq 2 \sqrt{ \pr{\sum_{k=1}^{t-1} \alpha_k} \pr{ R(\bw_1) - \inf_{\bw \in \sH} R(\bw) + \frac{\beta \sigma^2}{2} \sum_{k=1}^{t-1} \alpha_k^2} } + \sigma \sum_{k=1}^{t-1} \alpha_k~.
  \end{equation*}
\end{lemma}
\begin{proof}
First we perform the decomposition,
\begin{align}
 \E_S\br{ \sum_{k=1}^{t-1} \alpha_k \|\nabla f(\bw_{S,k}, z_k)\| } &=
\sum_{k=1}^{t-1} \alpha_k \E_{S}\br{\left\|  \nabla R(\bw_{S,k}) \right\|}
+ \sum_{k=1}^{t-1} \alpha_k \E_{S}\br{\left\|  \nabla f(\bw_{S,k}, z_k) - \nabla R(\bw_{S,k}) \right\|} \nonumber \\
&\leq  \sum_{k=1}^{t-1} \alpha_k \E_{S}\br{\left\|  \nabla R(\bw_{S,k}) \right\|} + \sigma \sum_{k=1}^{t-1} \alpha_k~. \label{eq:path_variance_decomposition}
\end{align}
Introduce
\[
Q_t := \sum_{k=1}^{t-1} \pr{\alpha_k - \frac{\alpha_k^2 \beta}{2}}~.
\]
Now we invoke the stationary-point argument to bound the first term above as
\begin{align}
\sum_{k=1}^{t-1} \alpha_k \E_{S}\br{ \sqrt{\left\|  \nabla R(\bw_k) \right\|^2} }
&\leq \sum_{k=1}^{t-1} \frac{\pr{1 - \frac{\alpha_k \beta}{2}}}{\pr{1 - \frac{\alpha_k \beta}{2}}} \cdot \alpha_k \sqrt{\E_{S}\br{ \left\|  \nabla R(\bw_k) \right\|^2 } } \tag{By Jensen's inequality}\\
&\leq 2 \sum_{k=1}^{t-1} \pr{\alpha_k - \frac{\alpha_k^2 \beta}{2}} \sqrt{\E_{S}\br{ \left\|  \nabla R(\bw_k) \right\|^2 } } \tag{Assuming that $\alpha_k \leq \frac{1}{\beta}$}\\
&=
\frac{2 Q_t}{Q_t} \sum_{k=1}^{t-1} \pr{\alpha_k - \frac{\alpha_k^2 \beta}{2}} \sqrt{\E_{S}\br{ \left\|  \nabla R(\bw_k) \right\|^2 } } \\
&\leq 2 \sqrt{Q_t} \sqrt{ \sum_{k=1}^{t-1} \pr{\alpha_k - \frac{\alpha_k^2 \beta}{2}} \E_{S}\br{ \left\|  \nabla R(\bw_k) \right\|^2 } } \tag{By Jensen's inequality}\\
    &\leq 2 \sqrt{Q_t} \sqrt{ R(\bw_1) - R(\bw_t) + \frac{\beta \sigma^2}{2} \sum_{k=1}^{t-1} \alpha_k^2 }~. \tag{By Lemma~\ref{lem:stationary_point}}
\end{align}

Combining this with~\eqref{eq:path_variance_decomposition} gives
\begin{align}
    \E_S\br{\sum_{k=1}^{t-1} \alpha_k \|\nabla f(\bw_{S,k}, z_k)\| }
&\leq 2 \sqrt{ \pr{\sum_{k=1}^{t-1} \alpha_k} \pr{ R(\bw_1) - \inf_{\bw \in \sH} R(\bw) + \frac{\beta \sigma^2}{2} \sum_{k=1}^{t-1} \alpha_k^2} } + \sigma \sum_{k=1}^{t-1} \alpha_k~,
\end{align}
which completes the proof.
\end{proof}

%
The following lemma is similar to Lemma~3.11 of~\cite{hardt2016train}, and is instrumental in bounding the stability of \ac{SGD}.
However, we make an adjustment and state it in expectation over the data.
Note that it does not require convexity of the loss function.
\begin{lemma}
\label{lem:stab_decompose}
Assume that the loss function $f(\cdot, z) \in [0, 1]$ is $L$-Lipschitz for all $z$.
Then, for every $t_0 \in \{0, 1, 2, \ldots m\}$ we have that,
\begin{align}
  &\E_{S,z}\E_A\br{f(\bw_{S,T}, z) - f(\bw_{\Srep,T}, z)} \label{eq:stab_decompose_lhs}\\
&\leq L \E_{S,z}\br{\E_{A}\br{\delta_T(S,z) \ \middle| \ \delta_{t_0}(S,z) = 0}}
+ \E_{S,A}\br{R(A_S)} \frac{t_0}{m}~. \label{eq:stab_decompose_1_E_cond}
\end{align}
\end{lemma}
\begin{proof}
We proceed with elementary decomposition, Lipschitzness of $f$, and using the fact that $f$ is non-negative to have that
\begin{align}
  f(\bw_{S,T}, z) - f(\bw_{\Srep,T}, z)
  &= \pr{f(\bw_{S,T}, z) - f(\bw_{\Srep,T}, z)} \ind{\delta_{t_0}(S,z) = 0} \label{eq:E_cond_stab_1} \\
  &+ \pr{f(\bw_{S,T}, z) - f(\bw_{\Srep,T}, z)} \ind{\delta_{t_0}(S,z) \neq 0} \nonumber \\
  &\leq L \delta_T(S,z) \ind{\delta_{t_0}(S,z) = 0}
  + f(\bw_{S,T}, z) \ind{\delta_{t_0}(S,z) \neq 0}~. \label{eq:E_cond_stab_2}
\end{align}
Taking expectation w.r.t. algorithm randomization, we get that
\begin{align}
  \E_A\br{f(\bw_{S,T}, z) - f(\bw_{\Srep,T}, z)}
  &\leq L \E_A\br{\delta_T(S,z) \ind{\delta_{t_0}(S,z) = 0} }\\
  &+ \E_A\br{f(\bw_{S,T}, z) \ind{\delta_{t_0}(S,z) \neq 0}}~. \label{eq:delta_neq_0_term}
\end{align}
Recall that $i \in [m]$ is the index where $S$ and $S\repi$ differ, and introduce a random variable $\tau_A$ taking on the index of the first time step where \ac{SGD} uses the example $z_i$ or a replacement $z$.
Note also that $\tau_A$ does not depend on the data.
When $\tau_A > t_0$, then it must be that $\delta_{t_0}(S,z) = 0$, because updates on both $S$ and $S\repi$ are identical until $t_0$.
A consequence of this is that $\ind{\delta_{t_0}(S,z) \neq 0} \leq \ind{\tau_A \leq t_0}$.
Thus the rightmost term in~\eqref{eq:delta_neq_0_term} is bounded as
\begin{align*}
  \E_A\br{f(\bw_{S,T}, z) \ind{\delta_{t_0}(S,z) \neq 0}} \leq \E_A\br{f(\bw_{S,T}, z) \ind{\tau_A \leq t_0} }~.
\end{align*}
Now, focus on the r.h.s.\ above.
Recall that we assume randomization by sampling from the uniform distribution over $[m]$ without replacement, and denote a realization by $\cbr{j_i}_{i=1}^m$.
Then, we can always express our randomization as permutation function $\pi_A(S) = \cbr{z_{j_i}}_{i=1}^m$.
In addition, introduce an algorithm $\GD : \sZ^m \mapsto \sH$, which is identical to $A$, except that it passes over the training set $S$ sequentially without randomization.
That said, we have that
\begin{align*}
  \E_A\br{f(\bw_{S,T}, z) \ind{\tau_A \leq t_0} }
= \E_A\br{f(\GD_{\pi_A(S)}, z) \ind{\tau_A \leq t_0} }~,
\end{align*}
and taking expectation over the data,
\begin{align*}
  \E_{S,z}\br{\E_A\br{f(\bw_{S,T}, z) \ind{\tau_A \leq t_0} }} = \E_A\br{\E_{S,z}\br{f(\GD_{\pi_A(S)}, z) } \ind{\tau_A \leq t_0} }~.
\end{align*}
Now observe that for any realization of $A$, $\E_{S,z}\br{f(\GD_{\pi_A(S)}, z) } = \E_A \E_{S,z}\br{f(A_S, z) }$ because expectation w.r.t.\ $S$ and $z$ does not change under our randomization
\footnote{Strictly speaking we could omit $\E_A[\cdot]$ and consider \emph{any} randomization by reshuffling, but we keep expectation for the sake of clarity.}.
Thus, we have that
\begin{align*}
  \E_A\br{\E_{S,z}\br{f(\GD_{\pi_A(S)}, z) } \ind{\tau_A \leq t_0} } = \E_{S,A}\br{R(A_S)} \P(\tau_A \leq t_0)~.
\end{align*}
Now assuming that $\tau_A$ is uniformly distributed over $[m]$ we have that
\[
\P\pr{\tau_A \leq t_0} = \frac{t_0}{m}~.
\]
Putting this together with~\eqref{eq:E_cond_stab_1} and~\eqref{eq:E_cond_stab_2}, we finally get that
\begin{align*}
  \E_{S,z}\E_A\br{f(\bw_{S,T}, z) - f(\bw_{\Srep,T}, z)}
  &\leq L \E_{S,z}\br{\E_{A}\br{\delta_T(S,z) \ind{\delta_{t_0}(S,z) = 0}}}
+ \E_{S,A}\br{R(A_S)} \frac{t_0}{m}\\
&\leq L \E_{S,z}\br{\E_{A}\br{\delta_T(S,z) \ \middle| \ \delta_{t_0}(S,z) = 0}}
+ \E_{S,A}\br{R(A_S)} \frac{t_0}{m}~.
\end{align*}
This completes the proof.
\end{proof}
We spend a moment to highlight the role of conditional expectation in~\eqref{eq:stab_decompose_1_E_cond}.
Observe that we could naively bound~\eqref{eq:stab_decompose_lhs} by the Lipschitzness of $f$, but Lemma~\ref{lem:stab_decompose} follows a more careful argument.
First note that $t_0$ is a free parameter.
The expected distance in~\eqref{eq:stab_decompose_1_E_cond} between \ac{SGD} outputs $\bw_{S,t}$ and $\bw_{\Srep,t}$ is conditioned on the fact that at step $t_0$ outputs of \ac{SGD} are still the same.
This means that the perturbed point is encountered after $t_0$.
Then, the conditional expectation should be a decreasing function of $t_0$: the later the perturbation occurs, the smaller deviation between $\bw_{S,t}$ and $\bw_{\Srep,t}$ we should expect.
Later we use this fact to minimize the bound~\eqref{eq:stab_decompose_1_E_cond} over $t_0$.
%
%








%
\subsection{Convex Losses}
\label{sec:proof_convex}
In this section we prove \stabname stability for loss functions that are non-negative, $\beta$-smooth, and convex.
\begin{theorem}
\label{thm:convex_stab_decomposed}
Assume that $f$ is convex, and that \ac{SGD}'s is ran with step sizes $\cbr{\alpha_t}_{t=1}^T$.
Then, for every $t_0 \in \{0, 1, 2, \ldots m\}$, \ac{SGD} is $\epsilon(\sD, \bw_1)$-\stabname stable with
\begin{align*}
\epsilon(\sD, \bw_1)
\leq \frac{2}{m} \sum_{t=t_0+1}^{T} \alpha_t \E_{S,z}\br{ \|\nabla f(\bw_t, z_{j_t})\| } + \E_{S,A}\br{\Risk(A_S)} \frac{t_0}{m}~.
\end{align*}
\end{theorem}
\begin{proof}
For brevity denote  $\Delta_t(S,z) := \E_A\br{\delta_t(S,z) \ | \ \delta_{t_0}(S,z) = 0}$.
We start by applying Lemma~\ref{lem:stab_decompose}:
\begin{align}
\E_{S,z}\E_A\br{f(\bw_{S,T}, z) - f(\bw_{\Srep,T}, z)}
\leq L \E_{S,z}\br{ \Delta_T(S,z) } + \E_{S,A}\br{\Risk(A_S)} \frac{t_0}{m}~. \label{eq:convex_t0_bound}
\end{align}
Our goal is to bound the first term on the r.h.s.\ as a decreasing function of $t_0$, so that eventually we can minimize the bound w.r.t.\ $t_0$.
At this point we focus on the first term, and the proof partially follows the outline of the proof of Theorem~3.7 in \cite{hardt2016train}.
The strategy will be to establish the bound on $\Delta_T(S,z)$ by using a recursive argument.
In fact we will state the bound on $\Delta_{t+1}(S,z)$ in terms of $\Delta_t(S,z)$ and then unravel the recursion.
Finally, we will take expectation w.r.t.\ the data after we obtain the bound by recursion.

To do so, we distinguish two cases: 1) \ac{SGD} encounters a perturbed point at step $t$, that is $t = i$, and 2) the current point is the same in $S$ and $\Srep$, so $t \neq i$.
For the first case, we will use data-dependent boundedness of the gradient update rule, Corollary~\ref{cor:G_boundedness}, that is
\[
\|G_t(\bw_{S,t}) - G_t(\bw_{\Srep,t})\| \leq \delta_t(S,z) + 2 \alpha_t \|\nabla f(\bw_{S,t}, z_{j_t})\|~.
\]
To handle the second case, we will use the expansiveness of the gradient update rule, Lemma~\ref{lem:worst-case-expansive}, which states that for convex loss functions, the gradient update rule is $1$-expansive, so $\delta_{t+1}(S,z) \leq \delta_t(S,z)$.
Considering both cases of example selection,
and noting that \ac{SGD} encounters the perturbation w.p.\ $\frac{1}{m}$,
we write $\E_A$ for a step $t$ as
\begin{align*}
\Delta_{t+1}(S,z) &\leq \pr{1 - \frac{1}{m}} \Delta_t(S,z)
+ \frac{1}{m} \pr{ \Delta_t(S,z) + 2 \alpha_t \|\nabla f(\bw_{S,t}, z_{j_t})\| }\\
&= \Delta_t(S,z) + \frac{2 \alpha_t \|\nabla f(\bw_{S,t}, z_{j_t})\|}{m}~.
\end{align*}
Unraveling the recursion from $T$ to $t_0$ and plugging the above into~\eqref{eq:convex_t0_bound} yields
\begin{align*}
\E_A\E_{S,z}[\delta_T(S,z)]
\leq \frac{2}{m} \sum_{t=t_0+1}^{T} \alpha_t \E_{S,z}\br{ \|\nabla f(\bw_t, z_{j_t})\| } + \E_{S,A}\br{\Risk(A_S)} \frac{t_0}{m}~.
\end{align*}
This completes the proof.
\end{proof}
Next statement is a simple consequence of Theorem~\ref{thm:convex_stab_decomposed} and Lemma~\ref{lem:E_path_bounded_by_R1_and_noise}.
\begin{proof}[Proof of Theorem~\ref{thm:sgd_stab_convex}.]
Consider Theorem~\ref{thm:convex_stab_decomposed} and set $t_0 = 0$.
\begin{align}
  \label{eq:convex_stab_bound_proof_eq1}
  \epsilon(\sD, \bw_1) &\leq \frac{2}{m} \sum_{t=1}^{T} \alpha_t \E_{S,z}\br{\|\nabla f(\bw_{S,t}, z_{j_t})\|}~.
\end{align}
Bounding the sum using Lemma~\ref{lem:E_path_bounded_by_R1_and_noise} recalling that $\alpha_t = c/ \sqrt{t}$, we get
\begin{align*}
  \E_S\br{\sum_{t=1}^{T} \alpha_t \|\nabla f(\bw_t, z_{j_t})\|}
  &\leq 2 \sqrt{ \pr{\sum_{t=1}^T \alpha_t} \pr{ R(\bw_1) - \Risk^{\star} + \frac{\beta \sigma^2}{2} \sum_{t=1}^T \alpha_t^2} }
  + \sigma \sum_{t=1}^T \alpha_t\\
  &\leq 2 \sqrt{2 c} \cdot \sqrt[4]{T} \cdot \sqrt{R(\bw_1) - \Risk^{\star} }
  + 2 c \sigma \pr{ \sqrt[4]{T} \sqrt{\frac{\beta}{2}} + \sqrt{T} }~.
\end{align*}
Combining above with~\eqref{eq:convex_stab_bound_proof_eq1} completes the proof.
\end{proof}
%
%
%
%
\subsection{Non-convex Losses}
Our proof of a stability bound for non-convex loss functions, Theorem~\ref{thm:sgd_stab_nonconvex} (in the submission file), follows a general outline of \citep[Theorem~3.8]{hardt2016train}.
Namely, 
the outputs of \ac{SGD} run on a training set $S$ and its perturbed version $\Srep$ will not differ too much,
because by the time a perturbation is encountered, the step size has already decayed enough.
So, on the one hand, stabilization is enforced by the diminishing the step size, and on the other hand, by how much updates expand the distance between the gradients after the perturbation.
Since \cite{hardt2016train} work with uniform stability, they capture the expansiveness of post-perturbation update by the Lipschitzness of the gradient.
In combination with a recursive argument, their bound has exponential dependency on the Lipschitz constant of the gradient.
We argue that the Lipschitz continuity of the gradient can be too pessimistic in general.
Instead, we rely on a local data-driven argument: considering that we initialize \ac{SGD} at point $\bw_1$, how much do updates expand the gradient under the distribution of interest?
The following crucial lemma characterizes such behavior in terms of the curvature at $\bw_1$.
%
\begin{lemma}
\label{lem:expansiveness_nonconvex}
Assume that the loss function $f(\cdot, z)$ is $\beta$-smooth and that its Hessian is $\rho$-Lipschitz.
Then,
\begin{align}
  \lf\| G_t(\bw_{S,t}) - G_t(\bw_{\Srep,t}) \rt\|
\leq \pr{1 + \alpha_t \xi_t(S,z) } \delta_t(S,z)
\end{align}
where
\begin{align*}
&\xi_t(S,z) := \lf\|\nabla^2 f(\bw_1, z_t) \rt\|_2
+ \frac{\rho}{2} \lf\| \sum_{k=1}^{t-1} \alpha_k \nabla f(\bw_{S,k}, z_k) \rt\|
+ \frac{\rho}{2} \lf\| \sum_{k=1}^{t-1} \alpha_k \nabla f(\bw_{\Srep,k}, z_{k'}) \rt\|~.
\end{align*}
Furthermore, for any $t \in [T]$,
\begin{align*}
   \E_{S,z}\br{\xi_t(S,z)} &\leq \E_{S,z}\br{\lf\|\nabla^2 f(\bw_1, z_t) \rt\|_2}\\
  &+ 2 \rho \sqrt{ \pr{R(\bw_1) - \Risk^{\star}} c (1 + \ln(T)) }\\
  &+ \rho \sigma \pr{ \sqrt{2 c \beta} + c (1 + \ln(T)) }~.
\end{align*}
\end{lemma}
%
%
%
\begin{proof}
Recall that the randomness of the algorithm is realized through sampling without replacement from the uniform distribution over $[m]$.
Apart from that we will not be concerned with the randomness of the algorithm, and given the set of random variables $\{j_i\}_{i=1}^m$, for brevity we will use indexing notation $z_1, z_2, \ldots, z_m$ to indicate $z_{j_1}, z_{j_2}, \ldots, z_{j_m}$.
Next, let $\Srep = \cbr{z_i'}_{i=1}^m$, and introduce a shorthand notation
$f_k(\bw) = f(\bw, z_k)$ and $f_{k'}(\bw) = f(\bw, z'_k)$.
We start by applying triangle inequality to get
\begin{align*}
\lf\| G_t(\bw_{S,t}) - G_t(\bw_{\Srep,t}) \rt\| \leq \|\bw_{S,t} - \bw_{\Srep,t}\|
+ \alpha_t \lf\| \nabla f_t(\bw_{S,t}) - \nabla f_t(\bw_{\Srep,t}) \rt\|~.
\end{align*}
In the following we will focus on the second term of r.h.s.\ above.
Given \ac{SGD} outputs $\bw_{S,t}$ and $\bw_{\Srep,t}$ with $t > i$, our goal here is to establish how much do gradients grow apart with every new update.
This behavior can be characterized assuming that gradient is Lipschitz continuous, however, we conduct a local analysis.
Specifically, we observe how much do updates expand gradients, given that we start at some point $\bw_1$ under the data-generating distribution.
So, instead of the Lipschitz constant, expansiveness rather depends on the curvature around $\bw_1$.
On the other hand, we are dealing with outputs at an arbitrary time step $t$, and therefore we first have to relate them to the initialization point $\bw_1$.
We do so by using the gradient update rule and telescopic sums, and conclude that this relationship is controlled by the sum of gradient norms along the update path.
We further establish that this sum is controlled by the risk of $\bw_1$ up to the noise of stochastic gradients, through stationary-point result of Lemma~\ref{lem:E_path_bounded_by_R1_and_noise}.
Thus, the proof consists of two parts: 1) Decomposition into curvature and gradients along the update path, and 2) bounding those gradients.
\paragraph{1) Decomposition.}
Introduce $\bdelta_t := \bw_{\Srep,t} - \bw_{S,t}$.
By Taylor theorem we get that
\begin{align*}
  &\nabla f_t(\bw_{S,t}) - \nabla f_t(\bw_{\Srep,t}) = \nabla^2 f_t(\bw_1) \bdelta_t
+ \int_0^1 \Big(\nabla^2 f_t(\bw_{S,t} + \tau \bdelta_t) - \nabla^2 f_t(\bw_1)\Big) \diff \tau \bdelta_t~.
\end{align*}
Taking norm on both sides, applying triangle inequality, Cauchy-Schwartz inequality, and assuming that Hessians are $\rho$-Lipschitz we obtain
\begin{align}
  &\|\nabla f_t(\bw_{S,t}) - \nabla f_t(\bw_{\Srep,t})\|
\leq \rho \int_0^1 \lf\| \bw_{S,t} - \bw_1 + \tau \bdelta_t \rt\| \diff \tau \|\bdelta_t\|
+ \lf\|\nabla^2 f_t(\bw_1) \rt\| \|\bdelta_t\|~.\label{eq:gradient_track_int}
\end{align}
\paragraph{2) Bounding gradients.}
Using telescoping sums and \ac{SGD} update rule we get that
\begin{align*}
\bw_{S,t} - \bw_1 + \tau \bdelta_t
&= \bw_{S,t} - \bw_1 + \tau \pr{\bw_{\Srep,t} - \bw_1 + \bw_1 - \bw_{S,t}} \\
&= \sum_{k=1}^{t-1} \pr{\bw_{S,k+1} - \bw_{S,k}}\\
&+ \tau \sum_{k=1}^{t-1} \pr{\bw_{\Srep,k+1} - \bw_{\Srep,k}}\\
&- \tau \sum_{k=1}^{t-1} \pr{\bw_{S,k+1} - \bw_{S,k}}\\
&= (\tau - 1) \sum_{k=1}^{t-1} \alpha_k \nabla f_k(\bw_{S,k}) - \tau \sum_{k=1}^{t-1} \alpha_k \nabla f_{k'}(\bw_{\Srep,k})~.
\end{align*}
Plugging above into the integral of~\eqref{eq:gradient_track_int} we have
\begin{align*}
&\int_0^1 \lf\| \sum_{k=1}^{t-1} \alpha_k \pr{ (\tau - 1) \nabla f_k(\bw_{S,k}) - \tau \nabla f_{k'}(\bw_{\Srep,k}) } \rt\| \diff \tau\\
  &\leq \frac{1}{2} \lf\| \sum_{k=1}^{t-1} \alpha_k \nabla f_k(\bw_{S,k}) \rt\| + \frac{1}{2} \lf\| \sum_{k=1}^{t-1} \alpha_k \nabla f_{k'}(\bw_{\Srep,k}) \rt\|\\
  &\leq \frac{1}{2} \sum_{k=1}^{t-1} \alpha_k \|\nabla f_k(\bw_{S,k})\| + \frac{1}{2} \sum_{k=1}^{t-1} \alpha_k \|\nabla f_{k'}(\bw_{\Srep,k}) \|~.
\end{align*}
Plugging this result back into~\eqref{eq:gradient_track_int} completes the proof of the first statement.
The second statement comes from Lemma~\ref{lem:E_path_bounded_by_R1_and_noise} with $\alpha_t = c/t$.
%
%
%
\end{proof}
Next, we need the following statement to prove our stability bound.
%
\begin{prop}[Bernstein-type inequality]
\label{prop:bernstein}
Let $Z$ be a zero-mean real-valued r.v., such that $|Z| \leq b$ and $\E[Z^2] \leq \sigma^2$.
Then for all $|c| \leq \frac{1}{2 b}$, we have that
$
\E\br{e^{c Z}} \leq e^{c^2 \sigma^2}~.
$
\end{prop}
\begin{proof}
Stated inequality is a consequence of a Bernstein-type inequality for moment generating functions, Theorem 2.10 in~\cite{boucheron2013concentration}.
Observe that zero-centered r.v. $Z$ bounded by $b$ satisfies Bernstein's condition, that is
\[
|\E[(Z - \E[Z])^q]| \leq \frac{q!}{2} \sigma^2 b^{k-2} \qquad \text{for all integers } q \geq 3~.
\]
This in turn satisfies condition for Bernstein-type inequality stating that
\[
\E\br{\exp\pr{c (Z - \E[Z])}} \leq \exp\pr{ \frac{c^2 \sigma^2/2}{1 - b |c|} }~.
\]
Choosing $|c| \leq \frac{1}{2 b}$ verifies the statement.
\end{proof}
%
%
Now we are ready to prove Theorem~\ref{thm:sgd_stab_nonconvex}, which bounds the $\epsilon(\sD, \bw_1)$-\stabname stability of \ac{SGD}.
\begin{proof}[Proof of Theorem~\ref{thm:sgd_stab_nonconvex}.]
%
For brevity denote
\[
\Riskend := \E_{S,A}\br{\Risk(A_S)}
\] and
\[
\Delta_t(S,z) := \E_A\br{\delta_t(S,z) \ | \ \delta_{t_0}(S,z) = 0}~.
\]
By Lemma~\ref{lem:stab_decompose}, for all $t_0 \in [m]$,
\begin{align}
\E_{S,z}\E_A\br{f(\bw_{S,T}, z) - f(\bw_{\Srep,T}, z)}
\leq L \E_{S,z}\br{ \Delta_T(S,z) } + \Riskend \frac{t_0}{m}~. \label{eq:nonconvex_t0_tuning_bound}
\end{align}
Most of the proof is dedicated to bounding the first term in~\eqref{eq:nonconvex_t0_tuning_bound}.
We deal with this similarly as in~\cite{hardt2016train}.
Specifically, we state the bound on $\Delta_T(S,z)$ by using a recursion.
In our case, however, we also have an expectation w.r.t.\ the data, and to avoid complications with dependencies, we first unroll the recursion for the random quantities, and only then take the expectation.
At this point the proof crucially relies on the product of exponentials arising from the recursion, and all relevant random quantities end up inside of them.
We alleviate this by Proposition~\ref{prop:bernstein}.
Finally, we conclude by minimizing
~\eqref{eq:nonconvex_t0_tuning_bound}
w.r.t.\ $t_0$.
Thus we have three steps: 1) recursion, 2) bounding $\E[\exp(\cdots)]$, and 3) tuning of $t_0$.
\paragraph{1) Recursion.}
We begin by stating the bound on $\Delta_T(S,z)$ by recursion.
Thus we will first state the bound on $\Delta_{t+1}(S,z)$ in terms of $\Delta_t(S,z)$, and other relevant quantities and then unravel the recursion.
As in the convex case,
we distinguish two cases: 1) \ac{SGD} encounters the perturbed point at step $t$, that is $t = i$, and 2) the current point is the same in $S$ and $\Srep$, so $t \neq i$.
For the first case, we will use worst-case boundedness of $G_t$, Corollary~\ref{cor:G_boundedness}, that is, $\|G_t(\bw_{S,t}) - G_t(\bw_{\Srep,t})\| \leq \delta_t(S,z) + 2 \alpha_t L~.$
To handle the second case we will use Lemma~\ref{lem:expansiveness_nonconvex}, namely,
\begin{equation*}
  \lf\| G_t(\bw_{S,t}) - G_t(\bw_{\Srep,t}) \rt\| \leq \pr{1 + \alpha_t \xi_t(S,z) } \delta_t(S,z)~.
\end{equation*}
In addition, as a safety measure we will also take into account that the gradient update rule is at most $(1 + \alpha_t \beta)$-expansive by Lemma~\ref{lem:worst-case-expansive}.
So we will work with the function $\psi_t(S,z) := \min\cbr{\xi_t(S,z), \beta}$ instead of $\xi_t(S,z)$.
and decompose the expectation w.r.t.\ $A$ for a step $t$.
Noting that \ac{SGD} encounters the perturbed example with probability $\frac{1}{m}$, 
\begin{align}
\Delta_{t+1}(S,z) &\leq \pr{1 - \frac{1}{m}} \pr{1 + \alpha_t \psi_t(S,z) } \Delta_t(S,z)
+ \frac{1}{m} \pr{ 2 \alpha_t L + \Delta_t(S,z) }~ \nonumber\\
&= \pr{1 + \pr{1 - \frac{1}{m}} \alpha_t \psi_t(S,z) } \Delta_t(S,z)
+ \frac{2 \alpha_t L}{m} \nonumber\\
&\leq \exp\pr{ \alpha_t \psi_t(S,z) } \Delta_t(S,z)
+ \frac{2 \alpha_t L}{m}~, \label{eq:nonconvex_recursion}
\end{align}
where the last inequality follows from $1 + x \leq \exp(x)$.
This inequality is not overly loose for $x \in [0,1]$, and, in our case it becomes instrumental in handling the recursion.

Now, observe that relation $x_{t+1} \leq a_t x_t + b_t$ with $x_{t_0} = 0$ unwinds from $T$ to $t_0$ as
$ x_T \leq \sum_{t=t_0+1}^T b_t \prod_{k=t+1}^T a_k$.
Consequently, having $\Delta_{t_0}(S,z) = 0$, we unwind~\eqref{eq:nonconvex_recursion} to get
\begin{align}
  \Delta_T(S,z)
\leq &\sum_{t=t_0+1}^T \pr{ \prod_{k=t+1}^T \exp\pr{\frac{c \psi_k(S,z)}{k} } } \frac{2 c L}{m t} \nonumber \\
= &\sum_{t=t_0+1}^T \exp \pr{ c \sum_{k=t+1}^T \frac{\psi_k(S,z)}{k} } \frac{2 c L}{m t}~. \label{eq:end_up_with_exponential}
\end{align}
\paragraph{2) Bounding $\E[\exp(\cdots)]$.}
We take expectation w.r.t.\ $S$ and $z$ on both sides and
focus on the expectation of the exponential in~\eqref{eq:end_up_with_exponential}.
First, introduce $\mu_k := \E_{S,z}[\psi_k(S,z)]$, and proceed as
\begin{align}
&\E_{S,z}\br{ \exp \pr{c \sum_{k=t+1}^T \frac{\psi_k(S,z)}{k} } }
= \E_{S,z}\br{ \exp \pr{c \sum_{k=t+1}^T \frac{\psi_k(S,z) - \mu_k}{k} } }
\exp\pr{ c \sum_{k=t+1}^T \frac{\mu_k}{k} }~. \label{eq:exp_split}
\end{align}
Observe that zero-mean version of $\psi_k(S,z)$ is bounded as
\begin{align*}
  \sum_{k=t+1}^T \frac{|\psi_k(S,z) - \mu_k|}{k} \leq 2 \beta \ln(T)~,
\end{align*}
and assume the setting of $c$ as $c \leq \frac{1}{2 (2 \beta \ln(T))^2}$.
By Proposition~\ref{prop:bernstein}, we have
\begin{align*}
  \E\br{ \exp \pr{c \sum_{k=t+1}^T \frac{\psi_k(S,z) - \mu_k}{k} } }
\leq &\exp \pr{c^2 \E\br{\pr{\sum_{k=t+1}^T \frac{\psi_k(S,z) - \mu_k}{k} }^2}}\\
= &\exp \pr{\frac{c}{2} \E\br{\pr{\frac{1}{2 \beta \ln(T)} \sum_{k=t+1}^T \frac{\psi_k(S,z) - \mu_k}{k} }^2}}\\
\leq &\exp \pr{ \frac{c}{2} \E\br{\abs{\sum_{k=t+1}^T \frac{\psi_k(S,z) - \mu_k}{k} }}}\\
\leq &\exp \pr{ \frac{c}{2} \sum_{k=t+1}^T \frac{\E\br{|\psi_k(S,z) - \mu_k|} }{k} }\\
\leq &\exp \pr{ c \sum_{k=t+1}^T \frac{ \mu_k }{k} }~.
\end{align*}
Getting back to~\eqref{eq:exp_split} we conclude that
\begin{align}
\label{eq:E_exp_bound}
 \E_{S,z}\br{ \exp \pr{c \sum_{k=t+1}^T \frac{\psi_k(S,z)}{k} } } \leq \exp \pr{ c \sum_{k=t+1}^T \frac{ 2 \mu_k }{k} }~.
\end{align}
%
Next, we give an upper-bound on $\mu_k$, that is
%
$
\mu_k \leq \min\cbr{\beta, \E_{S,z}[\xi_k(S,z)]}
$.
Finally, we bound $\E_{S,z}[\xi_k(S,z)]$ using the second result of Lemma~\ref{lem:expansiveness_nonconvex}, which holds for any $k \in [T]$,
to get that $\mu_k \leq \gamma$, with $\gamma$ defined in the statement of the theorem.
%
\paragraph{3) Tuning of $t_0$.}
Now we turn our attention back to~\eqref{eq:end_up_with_exponential}.
Considering that we took an expectation w.r.t.\ the data, we use~\eqref{eq:E_exp_bound} and the fact that $\mu_k \leq \gamma$ to get that
\begin{align*}
\E_{S,z}[\Delta_T(S,z)] &\leq \sum_{t=t_0+1}^T \exp \pr{ 2 c \gamma \sum_{k=t+1}^T \frac{ 1 }{k} }  \frac{2 c L}{m t}\\
&\leq \sum_{t=t_0+1}^T \exp \pr{2 c \gamma \ln\pr{\frac{T}{t}} } \frac{2 c L}{m t}\\
&= \frac{2 c L}{m} \pr{T^{2 c \gamma}} \sum_{t=t_0+1}^T t^{-2 c \gamma - 1}\\
&\leq \frac{1}{2 c \gamma }\frac{2 c L}{m} \pr{\frac{T}{t_0}}^{2 c \gamma}~.
\end{align*}
Plug the above into~\eqref{eq:nonconvex_t0_tuning_bound} to get
\begin{align}
\E_{S,z}\E_A\br{f(\bw_{S,T}, z) - f(\bw_{\Srep,T}, z)}
\leq \frac{L^2}{ \gamma m } \pr{\frac{T}{t_0}}^{2 c \gamma}  + r \frac{t_0}{m}~. \label{eq:bound_to_minimize_wrt_t0}
\end{align}
Let $q = 2 c \gamma$. Then, setting
\[
t_0 = \pr{\frac{2 c L^2}{r}}^{\frac{1}{1 + q}} T^{\frac{q}{1 + q}}
\]
minimizes~\eqref{eq:bound_to_minimize_wrt_t0}.
Plugging $t_0$ back we get that~\eqref{eq:bound_to_minimize_wrt_t0} equals to
\[
\frac{1 + \frac{1}{q}}{m} \pr{2 c L^2}^{\frac{1}{1 + q}} (r T)^{\frac{q}{1 + q}}~.
\]
%
%
This completes the proof.
\end{proof}
\subsubsection{Optimistic Rates for Learning with Non-convex Loss Functions}
Next we will prove an optimistic bound based on Theorem~\ref{thm:sgd_stab_nonconvex}, in other words, the bound that demonstrates fast convergence rate subject to the vanishing empirical risk.
First we will need the following technical statement.
\begin{lemma}{\citep[Lemma~7.2]{cucker2007learning}}
\label{lem:poly_root_bound}
Let $c_1, c_2, \ldots, c_l > 0$ and $s > q_1 > q_2 > \ldots > q_{l-1} > 0$.
Then the equation
\[
x^s - c_1 x^{q_1} - c_2 x^{q_2} - \cdots - c_{l-1} x^{q_{l-1}} - c_l = 0
\]
has a unique positive solution $x^{\star}$. In addition,
\[
x^{\star} \leq \max\left\{(l c_1)^\frac{1}{s-q_1}, (l c_2)^\frac{1}{s-q_2}, \cdots, (l c_{l-1})^\frac{1}{s-q_{l-1}}, (l c_l)^\frac{1}{s} \right\}.
\]
\end{lemma}
Next we prove a useful technical lemma similarly as in~\citep[Lemma~7]{orabona2014simultaneous}.
\begin{lemma}
\label{lem:nonparametric_optimistic_solve}
Let $a, c > 0$ and $0 < \alpha < 1$. Then the inequality
\[
x - a x^{\alpha} - c \leq 0
\]
implies
\[
x \leq \max\cbr{ 2^{\frac{\alpha}{1 - \alpha}} a^{\frac{1}{1 - \alpha}}, \pr{2 c}^{\alpha} a } + c~.
\]
\end{lemma}
\begin{proof}
Consider a function $h(x) = x - a x^{\alpha} - c$.
Applying Lemma~\ref{lem:poly_root_bound} with $s=1$, $l=2$, $c_1=a$, $c_2=c$, and $q_1 = \alpha$ we get that $h(x) = 0$ has a unique positive solution $x^{\star}$ and
\begin{equation}
\label{eq:xstar}
x^{\star} \leq \max\cbr{ (2 a)^{\frac{1}{1 - \alpha}}, 2 c }~.
\end{equation}
Moreover, the inequality $h(x) \leq 0$ is verified for $x=0$, and $\lim_{x \rightarrow +\infty} h(x) = +\infty$, so we have that $h(x) \leq 0$ implies $x \leq x^{\star}$.
Now, using this fact and the fact that $h(x^{\star}) = 0$, we have that
\[
x \leq x^{\star} = a \pr{x^{\star}}^{\alpha} + c~,
\]
and upper-bounding $x^{\star}$ by~\eqref{eq:xstar} we finally have
\[
x \leq a \max\cbr{ (2 a)^{\frac{\alpha}{1 - \alpha}}, \pr{2 c}^{\alpha} } + c~,
\]
which completes the proof.
\end{proof}
\begin{proof}[Proof of Corollary~\ref{cor:nonconvex_optimistic}.]
Consider Theorem~\ref{thm:sgd_stab_nonconvex} and observe that it verifies condition of Lemma~\ref{lem:nonparametric_optimistic_solve} with $x = \E_{S,A}\br{\Risk(A_S)}$, $c = \E_{S,A}\br{\Riskh_S(A_S)}$, $\alpha = \frac{c \gamma}{1 + c \gamma}$, and
\[
a = \frac{1 + \frac{1}{c \gamma}}{m} \pr{2 c L^2}^{\frac{1}{1 + c \gamma}} T^{\frac{c \gamma}{1 + c \gamma}}~.
\]
Note that $\alpha/(1-\alpha) = c \gamma$ and $1/(1-\alpha)=1+c \gamma$.
Then, we obtain that
\begin{align*}
&\E_{S,A}\br{\Risk(A_S) - \Riskh_S(A_S)}\\
&\leq \max\Bigg\{
2^{c \gamma} \pr{\frac{1 + \frac{1}{c \gamma}}{m}}^{1 + c \gamma} \pr{2 c L^2} T^{c \gamma},
\pr{2 \E_{S,A}\br{\Riskh_S(A_S)}}^{\frac{c \gamma}{1 + c \gamma}}
\pr{\frac{1 + \frac{1}{c \gamma}}{m} \pr{2 c L^2}^{\frac{1}{1 + c \gamma}} T^{\frac{c \gamma}{1 + c \gamma}}}
\Bigg\}\\
&= \max\Bigg\{
\pr{2 + \frac{2}{c \gamma}}^{1 + c \gamma} \pr{c L^2} \pr{\frac{T^{c \gamma}}{m^{1 + c \gamma}}},
\frac{1 + \frac{1}{c \gamma}}{m} \pr{2 c L^2}^{\frac{1}{1 + c \gamma}} \pr{2 \E_{S,A}\br{\Riskh_S(A_S)} \cdot T}^{\frac{c \gamma}{1 + c \gamma}}
\Bigg\}~.
\end{align*}
This completes the proof.
\end{proof}
\begin{proof}[Proof of Proposition~\ref{prop:nonconvex_transfer}]
Consider minimizing the bound given by Corollary~\ref{cor:non_convex_R1} (in the submission file) over a discrete set of source hypotheses $\cbr{\bw\src_k}_{k=1}^K$,
\begin{align}
&\min_{k \in [K]}\epsilon(\sD, \bw\src_k) \nonumber\\
&\leq \min_{k \in [K]} \scO\pr{ \frac{1 + \frac{1}{c \gamma_k}}{m}  \pr{ \Risk(\bw\src_k) \cdot T}^{\frac{c \gamma_k}{1 + c \gamma_k}} }~,
\label{eq:nonconvex_TL_bound}
\end{align}
and let
\begin{align*}
&\gamma_k = \scO\pr{ \E_{z \sim \sD}\br{ \|\nabla^2 f(\bw\src_k, z)\|_2 } + \sqrt{\Risk(\bw\src_k)} }~,\\
&\wh{\gamma}_k = \frac{1}{m} \sum_{i=1}^m \|\nabla^2 f(\bw\src_k, z_i)\|_2 + \sqrt{\Riskh_S(\bw\src_k)}~.
\end{align*}
By Hoeffding inequality, with high probability, we have that
$
| \gamma_k - \wh{\gamma}_k | \leq \scO\pr{ \frac{1}{\sqrt[4]{m}} }
$.
Now we further upper bound~\eqref{eq:nonconvex_TL_bound}
by upper bounding $\Risk(\bw\src_k)$ and apply union bound to get
\begin{align*}
&\min_{k \in [K]} \epsilon(\sD, \bw\src_k)\\
&\leq
\min_{k \in [K]} \scO\pr{ \pr{1 + \frac{1}{c \wh{\gamma}^-_k}}
\Riskh_S(\bw\src_k)^{\frac{c \wh{\gamma}^+_k}{1 + c \wh{\gamma}^+_k}}
\cdot
\frac{\sqrt{\log(K)}}{m^{\frac{1}{1 + c \wh{\gamma}^+_k}} } }~,
\end{align*}
where $\wh{\gamma}_k^{\pm} = \wh{\gamma}_k \pm \frac{1}{\sqrt[4]{m}}$.
This completes the proof.
\end{proof}

\end{document}